\RequirePackage{fix-cm}
\documentclass[smallextended]{svjour3}       % onecolumn (second format)
\smartqed  % flush right qed marks, e.g. at end of proof
\usepackage{graphicx}
\usepackage{comment}

\usepackage{amsmath, amssymb, mathrsfs}
\usepackage{amssymb, color}
\usepackage{mathpazo}

\newcommand{\norm}[1]{\left\lVert#1\right\rVert}

\newcommand{\rf}[1]{(\ref{#1})}
\def\lnfrac#1#2{\raise.7ex \hbox{\Small $#1$}%%
  \kern-.15em/\kern-.15em  \lower.2ex \hbox{\Small $#2$}}

\spnewtheorem{assumption}{Assumption}{\bf}{\it}

\numberwithin{equation}{section}
\begin{document}

\title{Power-law Dynamic arising from machine learning%\thanks{Grants or other notes
}


\author{Wei Chen$^{1,\dag}$, Weitao Du$^{2,\dag,*}$, Zhi-Ming Ma$^{3,\dag,*}$, Qi Meng$^{1,\dag}$
}


\institute{
$^1$ Microsoft Research Asia, wche@microsoft.com and Qi.Meng@microsoft.com; $^2$
University of Science and Technology of China,
duweitao@mail.ustc.edu.cn; $^3$ University of Chinese Academy of Sciences, mazm@amt.ac.cn \\ $\dag$: Alphabetical order
\\ *: Corresponding authors
}


\maketitle

\begin{abstract}
We study a kind of new SDE that was arisen  from the research on optimization in machine learning, we call it power-law dynamic because its stationary distribution  cannot have sub-Gaussian tail and obeys power-law. We prove that the power-law dynamic is ergodic with unique stationary distribution, provided the learning rate is small enough. We investigate its first exist time. In particular, we compare the exit times of the (continuous) power-law dynamic and its discretization. The comparison can help guide machine learning algorithm.

\keywords{machine learning\and stochastic gradient descent \and stochastic differential equation\and power-law dynamic}
\end{abstract}

\section{Introduction}

In the past ten years, we have witnessed the rapid development of deep machine learning technology. We successfully train deep neural networks (DNN) and achieve big breakthroughs in AI tasks, such as computer vision \cite{he2015delving,he2016deep,krizhevsky2012imagenet}, speech recognition \cite{oord2016wavenet,ren2019fastspeech,shen2018natural} and natural language processing \cite{he2016dual,sundermeyer2012lstm,vaswani2017attention}, etc.

Stochastic gradient descent (SGD) is a mainstream optimization algorithm in deep machine learning. Specifically, in each iteration, SGD randomly sample a mini batch of data and update the model by the stochastic gradient. For large DNN models, the gradient computation over each instance is costly. Thus, compared to gradient descent which updates the model by the gradient over the full batch data, SGD can train DNN much more efficiently. In addition, the gradient noise may help SGD to escape from local minima of the non-convex optimization landscape.

Researchers are investigating how the noise in SGD influences the optimization and generalization of deep learning. Recently, more and more work take SGD as the numerical discretization of the  stochastic differential equations (SDE) and investigate the dynamic behaviors of SGD by analyzing the SDE, including
the convergence rate \cite{he2018differential,li2017stochastic,rakhlin2012making},  the first exit time \cite{gurbuzbalaban2020heavy,meng22020dynamic,wu2018sgd,xie2020diffusion}, the PAC-Bayes generalization bound
\cite{he2019control,mou2017generalization,smith2017bayesian} and the optimal hyper-parameters \cite{he2019control,li2017stochastic}. Most of the results in this research line are derived from the dynamic  with state-independent noise, assuming that
the diffusion coefficient of SDE is a constant matrix independent with the state (i.e., model parameters in DNN). However, the covariance of the gradient noise in SGD does depend on the model parameters.

In our recent work \cite{meng22020dynamic,meng2020dynamic}, we studied the dynamic behavior of SGD with state-dependent noise.
We found that the covariance of the gradient noise of SGD in
the local region of local minima can be well approximated by a  quadratic function of the state. Then, we proposed to investigate the dynamic behavior of SGD by a stochastic differential equation (SDE) with a quadratic state-dependent diffusion coefficient. As shown in \cite{meng22020dynamic,meng2020dynamic}, the new SDE with quadratic diffusion coefficient can better matches the behavior of SGD compared with the SDE with constant diffusion coefficient.

In this paper, we study some mathematical properties of the new SDE with quadratic diffusion coefficient. After briefly introducing its machine learning background and investigating its preliminary properties (Section 2),  we show in Section 3 that the stationary distribution of this new SDE is a power-law distribution ( hence we call the corresponding dynamic a \emph{power-law dynamic} ), and the distribution possesses heavy-tailed property, which means that it cannot have sub-Gaussian tail. Employing coupling method, in Section 4 we prove that the power-law dynamic is ergodic with unique stationary distribution, provided the learning rate is small enough. In the last two sections we analyze the first exit time of the power-law dynamic. We obtain an  asymptotic order of the first exit time in Section 5, we then in Section 6 compare the exit times of the (continuous) power-law dynamic and its discretization. The comparison can help guide machine learning algorithm.

\section{Background and preliminaries on power-law dynamic}
\subsection{Background in Machine Learning}

Suppose that we have training data $S_n=\{(x_1,y_1),\cdots,(x_n,y_n)\}$ with inputs $\{x_j\}_{j=1}^n\in\mathbb{R}^{d_1\times n}$ and outputs $\{y_j\}_{j=1}^n\in\mathbb{R}^{d_2\times n}$. For a model $f_w(x):\mathbb{R}^{d_1}\rightarrow\mathbb{R}^{d_2}$ with parameter (vector) $w \in\mathbb{R}^{d}$, its loss over the training instance $(x_j,y_j)$ is $l(f_w(x_j),y_j)$, where $\ell(\cdot,\cdot)$ is the loss function. In machine learning, we are minimizing the empirical loss over the training data, i.e.,
\begin{align}\label{eq:2.1}
    \min_wL(w):=\frac{1}{n}\sum_{j=1}^n\ell(f_w(x_j),y_j).
\end{align} Stochastic gradient descent (SGD) and its variants are the mainstream approaches to minimize $L(w)$. In SGD, the update rule at the $k$-th iteration is
\begin{align}\label{eq:2.2}
    w_{k+1}=w_k-\eta\cdot \tilde{g}(w_k),
\end{align}
where $\eta$ denotes the learning rate,
\begin{align}\label{eq:2.3}
\tilde{g}(w):=\frac{1}{b}\sum_{j\in S_b}\nabla_w\ell(f_w(x_j,y_j))
\end{align}
 is the stochastic gradient, with $S_b$ being a random sampled subset of $S_n$ with size $b:=|S_b|$. In the literature, $S_b$ is called mini-batch.

We know that $\tilde{g}(w)$ is an unbiased estimator of the full gradient $\nabla L(w)$. The gap between the full gradient and the stochastic gradient, i.e.,
\begin{align}\label{eq:2.4}
R(w):= \nabla L(w)-\tilde{g}(w),
\end{align}
is called the gradient noise in SGD. In the literature, e.g.  \cite{li2017stochastic,meng22020dynamic,xie2020diffusion}, the gradient noise $R(w)$ is assumed to be drawn from Gaussian distribution  \footnote{Under mild conditions the assumption is approximately satisfied by the Central Limit Theorem.},that is, $R(w)\sim\mathcal{N}(0,\Sigma(R(w)))$, where $\Sigma(R(w))$ is the covariance matrix of $R(w).$  Denote $\Sigma(R(w))$ by $C(w),$ the update rule of SGD in Eq.(\ref{eq:2.2}) is then approximated by:
\begin{align}\label{eq:2.4}
    w_{k+1}=w_k-\eta\nabla L(w_k)+ \eta \xi_k,~~ \xi_k \sim\mathcal{N}(0, C(w_k)).
\end{align}
Further, for small enough learning rate $\eta$, Eq.(\ref{eq:2.4}) can be viewed
as the numerical discretization of the following stochastic differential equation (SDE) \cite{he2018differential,li2017stochastic,meng22020dynamic},
\begin{align}\label{eq:2.5}
    dw_t=-\nabla L(w_t)dt+\sqrt{\eta C(w_t)} dB_t,
\end{align}
where $B_t$ is the standard Brownian Motion in $\mathbb{R}^{d}$. This viewpoint enable the researchers to investigate the dynamic properties of SGD
by means of stochastic analysis. In this line, recent
 work studied the dynamic of SGD with the help of SDE. However,
most of the quantitative results in this line of work were obtained for the dynamics with state-independent noise.
 More precisely, the authors assumed that the covariance
 $C(w_t)$ in Eq.(\ref{eq:2.5}) is a constant matrix independent with the state $w_t$.  This assumption of constant diffusion coefficient simplifies the calculation and the corresponding analysis. But it is over simplified because the noise covariance in SGD does depend on the model parameters.

 In our recent work \cite{meng22020dynamic,meng2020dynamic}, we studied the dynamic behavior of SGD with state-dependent noise.
 The theoritical conduction and empirical observations of our research show that the covariance of the gradient noise of SGD in
the local region of  local minima can be well  approximated by a quadratic function of the state $w_t$ as briefly reviewed below.

Let $w^*$ be a local minimum of the (empirical/training) loss function defined in \rf{eq:2.1}. We assume that
the loss function in the local region of $w^*$ can be approximated by the second-order
Taylor expansion as
\begin{align}\label{eq:2.7}
    L(w)=L(w^*) + \nabla_wL(w^*)(w-w^*) +\frac{1}{2}(w-w^*)^TH(w-w^*),
\end{align}
where H is the Hessian matrix of loss at $w^*.$  Since $\nabla_wL(w^*)=0$ at the local minimum $w^*,$  (\ref {eq:2.7}) is reduced to
\begin{align}\label{eq:2.8}
    L(w)=L(w^*) +\frac{1}{2}(w-w^*)^TH(w-w^*),
\end{align}
Under the above setting, the full gradient of training loss is
\begin{align}\label{eq:2.8}
\nabla L(w)= H(w-w^*),
\end{align}
and the stochastic gradient (\ref {eq:2.3}) is
\begin{align}\label{eq:2.9}
\tilde{g}(w):= \tilde{g}(w^*) + \tilde{H}(w-w^*)
\end{align}
where $\tilde{g}(\cdot)$ and $\tilde{H}(\cdot)$ are the gradient and Hessian calculated by the minibatch. More explicitly, the $i$-th component of $\tilde{g}(w)$ is
\begin{align}\label{eq:2.?}
\tilde{g}_i(w) =  \tilde{g}_i(w^*) + \sum_{a=1}^d \tilde{H}_{ia}(w_a-w^*_a).
\end{align}

Assuming that $Cov(\tilde{g}_i(w^*),\tilde{H}_{jk})=0$ for
~$ i,j,k\in \{1,...,d\}$, \footnote{This assumption holds for additive noise case and the squared loss \cite{wei2020implicit}. Specifically, for $\ell(w)=(y-f_w(x))^2$, the gradient and Hessian are $g(w)=2f'_w(x)(y-f_w(x))$ and $H(w)=2(f'_w(x))^2+2f''_w(x)(y-f_w(w))\approx2(f'_w(x))^2$ . With additive noise, we have $y=f_{w^*}(x)+\epsilon$ where $\epsilon$ is white noise independent with the input $x$. Then, $g(w^*)=2f'_w(x)\epsilon$, $H(w^*)\approx 2(f'_{w^*}(x))^2$, and we have $Cov(\tilde{g}_i(w^*),\tilde{H}_{jk})\approx0$.} we have

\begin{align}\label{eq:2.?}
C(w)_{ij}=Cov (\tilde{g}_i(w), \tilde{g}_j(w))=\Sigma_{ij}+(w-w^*)^TA^{ij}(w-w^*),
\end{align}
where $\Sigma_{ij}=Cov(\tilde{g}_i(w^*),\tilde{g}_j(w^*) ),$ ~~ $A^{ij}$ is a $d \times d$ matrix with elements $A^{ij}_{ab}= Cov(\tilde{H}_{ia},\tilde{H}_{jb}).$

Thus, we can convert $C(w)$ into an analytic tractable form as follows.
\begin{align}\label{new0}
C(w)=\Sigma_g(I +(w-w^*)^T\Sigma_H(w-w^*))
\end{align}
where $\Sigma_g$ and $\Sigma_H$ are positive definite matrix. The empirical observations in \cite{meng22020dynamic,meng2020dynamic} is consistent with the covariance structure
\rf{new0}.  Thus the SDE \rf{eq:2.5} takes the form
\begin{equation} \label{full}
dw_t = - H(w_t - w^*)dt + \sqrt{\eta C(w_t)}dB_t,
\end{equation}
where $C(w)$ is given by \rf{new0}. We call the dynamic driven by
\rf{full} a  \emph{power-law dynamic} because its stationary distribution obeys power-law  (see Theorem \ref{thm:sta} below). As shown in \cite{meng22020dynamic,meng2020dynamic}, power-law dynamic can better match the behavior of SGD compared to the SDE with constant diffusion coefficient.

\subsection{Preliminaries on power-law dynamic}

For the power-law dynamic \rf{full},
the infinitesimal generator exists and has the following form:
$$\mathcal{A}= \sum_i\sum_j\frac{\eta}{2}(\Sigma_g)_{ij}(1 + w^T\Sigma_H w)\frac{\partial}{\partial w^i}\frac{\partial}{\partial w^j} - \sum_i H_{ij}w^j \frac{\partial}{\partial w^i}.$$
We will use the infinitesimal generator to specify a coupling in the subsequent sections.
Write $v_t = w_t - w^*$, then
\begin{equation} \label{new1}
dv_t = - H v_t dt + \sqrt{\eta C(v_t)}dB_t,
\end{equation}
where $C(v) = \Sigma_g (1+ v^T \Sigma_H v)$ (Comparing with \rf{full}, here we slightly abused the notation $C$).

In machine learning, we often assume that the dynamic in (\ref{new1}) can be decoupled \cite{meng2020dynamic,xie2020diffusion,zhang2019algorithmic}. More explicitly, we assume that $\Sigma_g,$  $\Sigma_H$ and $H$ are codiagonalizable by an orthogonal matrix $Q$, then under the affine transformation $v_t = Q( w_t - w^*_t)$, (\ref{new1}) is decoupled and can be written as
\begin{equation}\label{sim}dv_t = -h_i v_t^i dt + \sqrt{\eta \sigma_i + \eta \rho_i (v_t^i)^2}dB^i_t,\ \ i \in \{1,\dots,d\},\end{equation}
where $\sigma_i$, $\rho_i$ are positive constants.\footnote{This decoupling property is empirically observed in machine learning, i.e., the directions of eigenvectors of the hessian matrix and the gradient covariance matrix are often nearly coincide at the minimum \cite{xie2020diffusion}. An explanation of this  phenomenon  is that under expectations the Hessian equals to Fisher information \cite{jastrzkebski2017three,xie2020diffusion,zhu2019anisotropic}. }

Following the convention of probabilistic literature, in what follows we shall write
\begin{equation} \label{new2}
 \mu(v_t)= - Hv_t
\end{equation}
\begin{equation}\label{sim3}
\sigma^2 (v_t)= \eta C(v_t).
\end{equation}

 Suppose $x,y \in \mathbb{R}^d$ , by the mean value theorem, we have the following inequality,
\begin{equation} \label{in}
|\sqrt{a+bx^2} - \sqrt{a+by^2}| \leq \sqrt{b}|x-y|.
\end{equation}
Then, it is easy to check that that both $\mu(\cdot)$ and $\sigma (\cdot)$ are local Lipschitz and
have linear growth.
Therefore, by standard theory of stochastic differential equations, the SDE  (\ref{new1}) has a unique strong solution $v(t),$  which has continuous paths and possesses strong Markov property.

Consider the decoupled dynamic in \rf{sim},  we use the fact that as $|x_i| \rightarrow \infty$,
$$|\frac{\mu_i(x_i)}{\sigma_i^2 (x_i)}| = |\frac{h_i x_i}{\eta \sigma_i + \eta \rho_i x_i^2}| \sim O(\frac{1}{|x_i|}).$$
Then, for any fixed $x_0,$ we have
$$\int_{-\infty}^{x_0} \exp (-\int_{x_0}^{x} \frac{2\mu_i(s)}{\sigma_i^2 (s)}ds)(\int_{x}^{x_0}\frac{\exp(\int_{x_0}^y \frac{2\mu_i(s)}{\sigma_i^2 (s)}ds)}{\sigma_i^2 (y)}dy)dx = \infty,$$
and
$$\int_{x_0}^{\infty} \exp (-\int_{x_0}^{x} \frac{2\mu(s)}{\sigma^2 (s)}ds)(\int_{x_0}^{x}\frac{\exp(\int_{x_0}^y \frac{2\mu(s)}{\sigma^2 (s)}ds)}{\sigma^2 (y)}dy)dx = \infty,$$
which implies that each component of $v_t$ will not blow up in finite time.

To conclude, the stochastic differential equation (\ref{sim}) admits a unique strong solution $v(t),$  which has continuous paths and will not blow up in finite time. In subsequent sections we shall study more properties of the dynamic $v(t).$

\section{ Property of the stationary distribution}\label{St}

In this section, we show  that the stationary distribution of the SDE \rf{new1} possesses heavy-tailed property, and its decoupled form is a product of power-law distributions. The existence and  uniqueness of the stationary distribution will be given in the next section.

Let $Q$ be an orthogonal matrix such that $H' = QHQ^T$ is a diagonal matrix. Then
\begin{align}\label{new4}
d (Q v_t) = - H'Qv_tdt + \sqrt{1+(Qv_t)^T \tilde{\Sigma}_HQv_t}\cdot \sqrt{\eta\tilde{\Sigma}_g}d\tilde{B}_t,
\end{align}
where $\tilde{\Sigma}_H = Q\Sigma_H Q^T$, $\tilde{\Sigma}_g = Q\Sigma_g Q^T.$  Note that  $\tilde{B}_t = Q B_t$ is still a Brownian motion. \rf{new4} is just the power-law dynamic  \rf{new1} under a new orthogonal coordinate system, so we will abuse the notation and denote the transformed dynamic by $v_t$ as well.
Since we care about the tail behavior of the power-law dynamic, we show first  that $v_t$ does not have finite higher moments as $t \rightarrow \infty$. This implies that $v_t$ cannot have exponential decay on the tail.

\begin{theorem} \label{thm:sta}
(i) We can find $m \ge 2$ such that the moments  of the power-law dynamic \rf{new4} of order greater than m will explode as the time $t \rightarrow \infty$. \\
(ii) For the decoupled case in \rf{sim}, the probabilistic density of the stationary distribution is a product of power-law distributions (the terminology follows from \cite{zhou}) as below:
\begin{align}\label{power}
p(x) = \frac{1}{Z} \prod_{i=1}^{d}(1+ \frac{\rho_i}{\sigma_i}x^2)^{\kappa_i},
\end{align}
where $\kappa_i = -\frac{\eta \rho_i + h_i}{\eta \rho_i}$ and $Z$ is the normalization constant.
\end{theorem}

\begin{proof}
(i)~~Denote the 2k-th moment of $v_t$ as $m_{2k}(t):= \sum_i \mathbb{E} (v_t^i)^{2k}$. Then $m_0(t) = \mathbb{E}[v_t^0] = 1$. By Ito's formula, we have
$$d \sum_i \mathbb{E} (v_t^i)^{2k} = \sum_i \{-H'_{ii} (v_t^i)^{2k} + k(2k-1)(1+ v^T\tilde{\Sigma_H}v)\tilde{\Sigma_g}_{ii} \mathbb{E}(v_t^i)^{2k-2}\}.$$
Let $h_{max}$ be the maximal diagonal element of $H'$ and $g_{min}$ be the minimal diagonal element of $\tilde{\Sigma}_g$, then we get the recursion inequality (note that it may not hold for the odd degree moments):
\begin{equation}
d m_{2k}(t) \ge (-k h_{max} + k(2k-1)g_{min}H_{min})m_{2k}(t) + k(2k-1)g_{min}m_{2k-2}(t),
\end{equation}
where $H_{min}$ is the minimal eigenvalue of the positive definite matrix $\tilde{\Sigma}_H$.
Let $a_k = -k h_{max} + k(2k-1)g_{min}H_{min}$ and $b_k = k(2k-1)g_{min}$, then
$$m_{2k}(t) \ge e^{a_k t} (x_0^{2k}+ \int_0^t R_k(s)\exp(-a_k s)ds),$$
the reminder term is defined by $R_k(s) : = k(2k-1)g_{min}m_{2k-2}(s)$. From the above relation, we can prove the following inequality by induction:
\begin{equation}
d m_{2k}(t) \ge \sum_{i=0}^k c_k^i \exp (a_k^i t).
\end{equation}
By tracking the related coefficients carefully, it is not difficult to find the recurrence relations for $c_k^i$. For example,
$$c_k^k = m_{2k}(0) - \sum_{i=0}^{k-1}c_k^i.$$
Since $g_{min}H_{min}$ is positive, $a_i$ becomes positive when i is large. From this fact, we can always find a $k$ such that $$\lim_{t \rightarrow \infty} m_{2k} (t) = \infty ,$$ which means the moment generating function of the stationary distribution blows up.  Therefore, the stationary distribution of $v_t$ cannot have sub-Gaussian tail.

(ii)~Now we turn to the decoupled case of equation \rf{sim}, since each coordinate is self-dependent, we know that the probabilistic density is of the product form. To investigate the probabilistic density $p(t,x)$ of one fixed coordinate $v_t^i$, we need to study the backward Kolmogolov equation satisfied by $p(t,x)$. Since $\mu(x)$ and $\sigma(x)$ have linear growth, we have
$$\frac{\partial p}{\partial t} = \frac{1}{2} \Delta (\sigma^2 p) - \nabla (\mu p) = \frac{1}{2} \Delta [(\eta \sigma_i + \eta \rho_i x^2)p] + \nabla[Hx p].$$
Here we adopt an idea from the statistical physics literature \cite{zhou} as in the machine learning area (see \cite{meng22020dynamic}), we first transform the Kolmogolov equation into the Smoluchowski form:
\begin{align} \label{fok}
\frac{1}{2} \Delta [(\eta \sigma_i + &\eta \rho_i x^2)p] + \nabla[h_i x p] = \nabla \cdot [\eta \rho_i x p + \frac{\sigma^2}{2} \nabla p + h_i xp]\\ \nonumber
& = \nabla[(h_ix + \eta\rho_i x) p] + \nabla[\frac{\sigma^2}{2} \nabla p].
\end{align}
Let $\frac{\partial U}{\partial x}: = Hx + \eta\rho_i x$, then the fluctuation-dissipation relation of $\sigma^2$ and $U$ in \cite{zhou} is satisfied with
$$\kappa_i = -\frac{\eta \rho_i + h_i}{\eta \rho_i}.$$
Let $p_s : = \lim_{t \rightarrow \infty}p(t,x)$, then the stationary distribution satisfies the power law:
\begin{equation} \label{power}
p_s(x) = \frac{1}{Z}(1+ \frac{\rho_i}{\sigma_i}x^2)^{\kappa_i}.
\end{equation}
The proof is completed.
\end{proof}

\begin{remark}

(i)
In \cite{zhou}, the tail index $\kappa_i$ ( depending on the hyper-parameter $\eta$ ) plays an important role in locating the large learning rate region. Observe that when $\kappa_i \ge -\frac{1}{2}$, the variance of $p_s$ is infinite.

(ii) Another way to view the power-law dynamic is to apply the results in the groundbreaking article \cite{ma}. Roughly speaking, the authors of \cite{ma} gave a complete classification in the Fourier space with a determined stationary distribution. Following the notation in  \cite{ma}, suppose we write the SDE in the following form:
$$dz = f(z)dt + \sqrt{2D(z)}dB(t),$$
where D(z) is a positive semi-definite diffusion matrix (a Riemannian metric). Suppose the stationary distribution $p_s(z) \propto \exp(-H(z))$, then the drift term $f(z)$ must satisfy:
$$f(z) = -[D(z) + Q(z)]\nabla H(z) + \Gamma(z),$$
where $Q(z)$ is an arbitrary skew-symmetric matrix (a symplectic form)  and $\Gamma(z)$ is defined by
$$\Gamma_i(z) = \sum_{j=1}^{d} \frac{\partial}{\partial z_j}(D_{ij}(z)+ Q_{ij}(z)).$$
When $d=1$, due to the skew-symmetry, $Q(z) \equiv 0$. If the stationary distribution is given by \rf{power}, $H(z) = \kappa \ln (1+ \frac{\rho}{\sigma}z^2)$. Thus,
$$\nabla H(z) = \frac{-\kappa \rho z}{\sigma + \rho z^2}.$$
We get that
$$f(z) = 2(1+\kappa)\eta\rho z.$$
In this way, we automatically obtain the fluctuation-dissipation relation.
\end{remark}
\section{Existence and uniqueness of the stationary distribution}

In this section, we shall prove that the power-law dynamic is   ergodic with unique stationary distribution, provided the learning rate $\eta$ is small enough (see theorem \ref{thm:uni} (ii) below). Note that unlike Langevin dynamics, we have a state-dependent diffusion term in the power-law dynamic and its stationary distribution does not have a sub-Gaussian tail, which makes the diffusion process break the log-sobolev inequality condition. Instead of treating $v_t$ as a gradient flow, we shall use  coupling method to bound the convergence of $v_t$ to its stationary distribution.

Let the drift vector $\mu(x) = - Hx$ and the diffusion matrix $\sigma^2(x) = \eta C(x)$, where $x \in \mathbb{R}^d,$
 be defined as in \rf{new2} and \rf{sim3} respectively. We set
 \begin{align}
& \theta:= \inf_{x,y}\{-<\mu(x) - \mu(y), x-y>/\norm{x-y}^2\},\\
&\lambda:=\sup_{x,y}\{ \max_i \sum_{1 \leq j \leq d}(\sigma_{ij}(x) - \sigma_{ij}(y))^2 /\norm{x-y}^2\}.
\end{align}
\begin{theorem} \label{thm:uni} (i)
Let $p(t,x,\cdot)$ be the transition probability of the power-law dynamic driven by \rf{new1}, we have
\begin{equation}\label{new5}
\mathbb{W}_2 (p(t,x,\cdot),p(t,y,\cdot)) \leq \parallel x-y \parallel e^{(d\cdot\lambda-\theta)t},
\end{equation}
where $\mathbb{W}_2 (\cdot,\cdot)$ is the Wasserstein distance between two probability distributions.

(ii) Employing the notations used in the previous section, we write $h_{min}$ for the minimal diagonal element of the matrix $H'$,  $g_{max}$ for the the maximal element of $\sqrt{\tilde{\Sigma_g}}$, and $H_{sum}$~ for the sum of the eigenvalues of ~$\tilde{\Sigma_H}$.
Suppose that
\begin{equation}\label{new6}
\eta < \frac{h_{min}}{d^2\cdot g_{max}^2 H_{sum}},
\end{equation}
then the power-law dynamic in \rf{new1} is  ergodic and its stationary distribution is unique.
\end{theorem}

\begin{proof}

(i) We shall employ the coupling method of Markov processes  in this proof and in the rest of this paper. The reader may refer to Chapter 2 of \cite{chen2006eigenvalues}, especially page 24 and Example 2.16, for the relevant contents. Recall that every infinitesimal generator of an $\mathbb{R}^{d}$-valued diffusion process has the form $\mathcal{A}_s = \sum_x \alpha(x) \frac{\partial^2}{\partial x^2} + \sum_{x} \beta(x)\frac{\partial}{\partial x} $. To specify a coupling between two power-law dynamics starting from different points of $\mathbb{R}^{d},$  we  define a coupling infinitesimal generator $\mathcal{A}_s(x,y)$,  $(x,y) \in \mathbb{R}^{2d},$ as follows:
$$\alpha_s(x,y) = \begin{pmatrix} \sigma(x)\sigma(x)^T, & \sigma(x)\sigma(y)^T \\ \sigma(y)\sigma(x)^T, & \sigma(y)\sigma(y)^T \end{pmatrix} ,\ \ \beta_s(x,y)=\begin{pmatrix} \mu(x) \\ \mu(y) \end{pmatrix}, $$
where $\sigma(\cdot)$ and $\mu(\cdot)$ are specified by \rf{new2} and \rf{sim3} respectively,
 $\alpha_s(x,y)$ corresponds to the second order differentiation and $\beta_s(x,y)$ corresponds to the first order differentiation.

Let $r(x,y) = \norm{x-y}^2$ and let $\mathcal{A}_s$ act on $r(x,y)$, we get
\begin{align*}
\mathcal{A}_s r(x,y) &= 2<\mu(x) - \mu(y), x-y> + \sum_i \sum_j (\sigma_{ij}(x) - \sigma_{ij}(y))^2\\
& \leq -2\theta\norm{x-y}^2 + 2d\lambda\norm{x-y}^2\\
& \leq c r(x,y),
\end{align*}
where $c := 2d\lambda - 2\theta$. Denote by $X_t$ the dynamic starting at $x$ and $Y_t$ the dynamic starting at $y$, by Ito's formula, we have
$$\frac{d \mathbb{E}r(X_t , Y_t)}{dt} \leq c \mathbb{E}r(X_t , Y_t).$$
Applying Gronwall¡¯s inequality, we get
$$\mathbb{E}r(X_t , Y_t) \leq r(x,y)e^{ct},$$
which implies that
$$\mathbb{W}_2 (p(t,x,\cdot),p(t,y,\cdot)) \leq \sqrt{r(x,y)}e^{ct/2}=\norm {x-y}  e^{(d\cdot\lambda-\theta)t},$$
verifying \rf{new5}.

(ii) In view of \rf{new5}, we need only to check that if \rf{new6} holds, then\\ $(d\cdot\lambda-\theta)< 0.$ We have
$$<\mu(x) - \mu(y), x-y> = -\sum_i H'_i(x_i - y_i)^2 \leq - h_{min}\norm{x-y}^2,$$
therefore
\begin{equation}\label{new7}
\theta \geq h_{min}.
\end{equation}
On the other hand, let $g_{max}$ be the maximal element of $\sqrt{\tilde{\Sigma_g}}$, then for all  $i,$
\begin{align*}
\sum_{1 \leq j \leq d}(\sigma_{ij}(x) - \sigma_{ij}(y))^2 &\leq \eta \cdot g_{max}^2 (\sqrt{1+ x^T\tilde{\Sigma_H}x} - \sqrt{1+ y^T\tilde{\Sigma_H}y})^2.
\end{align*}
Since $\norm{x-y}$ is preserved under orthogonal transformation,  then by the mean value theorem and Cauchy inequality, we can find  $ (\theta_1,\dots,\theta_d)$, such that
\begin{align*}
(\sigma_{ij}(x) - \sigma_{ij}(y))^2 &\leq \eta \cdot g_{max}^2 | (\frac{h_1 \theta_1}{\sqrt{1 + \theta^T \tilde{\Sigma_H} \theta}},\dots,\frac{h_d \theta_d}{\sqrt{1 + \theta^T \tilde{\Sigma_H} \theta}}) \\
&\cdot (x_1 - y_1, \dots, x_d - y_d)|^2
 \leq  \eta \cdot g_{max}^2 H_{sum} \norm{x-y}^2,
\end{align*}
where $\{h_i\}$ denote the eigenvalues of $\tilde{\Sigma_H},$ and  $H_{sum}$ denotes the sum of the eigenvalues. Thus,
$$\max_i \sum_{1 \leq j \leq d}(\sigma_{ij}(x) - \sigma_{ij}(y))^2 \leq d\cdot \eta \cdot g_{max}^2  H_{sum} \norm{x-y}^2.$$
Consequently,
\begin{equation}\label{new8}
\lambda \leq d\cdot\eta \cdot g_{max}^2 H_{sum}.
\end{equation}
Combining \rf{new7} and \rf{new8}, we see that \rf{new6} implies $(d\cdot\lambda-\theta)< 0.$  Therefore Assertion (ii) holds by the
virtue of \rf{new5}. The proof is completed.
\end{proof}

\begin{remark}
If we restrict ourselves in the decoupled case \rf{sim}, we can get the exponential convergence to stationary distribution under much weaker condition of $\eta$.  Notice that now $\sigma(x)$ is a diagonal matrix.
Using short hand writing $(\sigma_{ii}(x) - \sigma_{ii}(y))^2$ for $\sum_{1 \leq j \leq d}(\sigma_{ij}(x) - \sigma_{ij}(y))^2 ,$ we have

\begin{align*}
(\sigma_{ii}(x) - \sigma_{ii}(y))^2 &\leq (\sqrt{\eta(\sigma_i + \rho_i(x^i)^2)}- \sqrt{\eta(\sigma_i + \rho_i(y^i)^2)})^2 \\
& \leq \eta \rho_i (x_i - y_i)^2.
\end{align*}
Then, we have
\begin{align*}
L_s r(x,y) &= \sum_i (\sigma_{ii}(x) - \sigma_{ii}(y))^2- \sqrt{\eta(\sigma_1 + \rho_i(y^i)^2)})^2 - h_i(x^i -y^i)^2]\\
& \leq \sum_i (\eta \rho_i - h_i)(x^i - y^i)^2\\
& \leq c_sr(x,y),
\end{align*}
where $c_s := \max_i [\eta \rho_i - h_i]$, which does not involve the dimension $d$.
\end{remark}

\section{ First exit time: asymptotic order}

From now on, we investigate  the first exit time from a ball of the power-law dynamic, which is an important issue in  machine learning.   By leveraging the transition rate results from the large deviation theory (see e.g. \cite{kraaij2019classical}), in this section we obtain an asymptotic order of the first exit time for the decoupled power-law dynamic.

\begin{theorem} \label{escaping}
 Suppose $0$ is the only local minimum of the loss function inside $B(0,r)$. Let $\tau_r^x(\eta)$ be the first exit time from $B(0,r)$ of the decoupled power-law dynamic in (\ref{sim}), with learning rate $\eta,$ starting at $x\in B(0,r)$, then
\begin{equation} \label{div}
\lim_{\eta \rightarrow 0} \eta \log \mathbb{E} \tau^x_r(\eta) = C \cdot \inf_{\zeta=(\zeta^1,\dots,\zeta^d) \in \partial B(0,r)} \sum_i - \frac{h_i}{\rho_i} \log [\sigma_i + \rho_i (\zeta^i)^2],\end{equation}
where $C$ is a prefactor to be determined.  \\
When $d=1$, we have an explicit expression of the first exit time from an interval $[a,b]$  starting at $x\in (a,b)$:
\begin{equation}
\mathbb{E} \tau_x = g(x) := 2\int_x^b  \frac{e^{\phi(y)}}{\sigma^2(y)} dy\int_a^y e^{-\phi(z)}dz,
\end{equation}
where $\phi = 2-\kappa \ln (1+\frac{\rho}{\sigma}x^2)$ and $\kappa = - \frac{\eta \rho + h}{\eta \rho}$.
\end{theorem}

\begin{proof}
Let $\tau$ be a stopping time with finite expectation and let $\mathcal{A}$ be the infinitesimal generator of $v_t$, then recall that Dynkin's formula tells us:
$$\mathbb{E}[f(v_{\tau})] = f(x) + \mathbb{E}[\int_{0}^{\tau}\mathcal{A}f(v_s)ds],\ \ \ f \in \mathcal{C}_0^2(\mathbb{R}^d),$$
where $v_0 = x$. Suppose $f$ solves the following boundary problem:
\begin{equation}
\left\{
\begin{array}{l}
\mathcal{A}f(x)=-1,\ \ x \in B(0,r), \\
f(x) = 0,\ \ x \in \partial B(0,r),
\end{array}
\right.
\end{equation}
then $\mathbb{E}\tau^x_r = f(x)$, where $\tau^x_r$ denote the first exit time of $v_t$ starting at $x$ from the ball $B(0,r)$.

 We consider first the situation of $d=1$, let $\tau^x_{(a,b)}(\eta)$ be the first exit time  of $v(t)$ from an interval $[a,b]$ starting at $x \in (a,b)$. Note that the diffusion coefficient function $\sigma(x) = \sqrt{\eta\sigma + \eta\rho x^2} > 0$, then by Dynkin's formula, $\mathbb{E}\tau^x_{(a,b)}(\eta) = g(x)$, where $g(x)$ solves the following second order ODE:
\begin{equation} \label{bou}
\left\{
\begin{array}{l}
\mathcal{A}_1 g(x)=-1,\ \ x \in (a,b), \\
g(x) = 0,\ \ x \in \{a,b\},
\end{array}
\right.
\end{equation}
where $\mathcal{A}_1 = - h x \frac{\partial}{\partial x} + (\eta\sigma + \eta\rho x^2)\frac{\partial^2}{\partial x^2}$ is the infinitesimal generator of the one dimensional power-law diffusion. Now we introduce the integration factor $\phi (x): = -\kappa \ln (1+\frac{\rho}{\sigma}x^2)$, following \rf{fok},
\begin{equation} \label{su}
\nabla[(h x + \eta\rho x) f(x)] + \nabla[\frac{\sigma^2}{2} \nabla f(x)]= \frac{1}{2}\nabla[\sigma^2 e^{-\phi}\nabla(e^{\phi}f(x))] = 0.
\end{equation}
Then,
\begin{align*}
\mathcal{A}_1 f(x) &= \frac{1}{2}\nabla[\sigma^2 e^{-\phi}\nabla(e^{\phi}f(x))] - \nabla[\eta\rho x f(x)] + h f(x) - \eta\rho x \nabla f(x) - 2 \nabla[h x f(x)] \\
& = \frac{1}{2}\nabla[\sigma^2 e^{-\phi}\nabla(e^{\phi}f(x))] - \frac{1}{2} \nabla[\dot{\phi} \sigma^2(x)f(x)]- \frac{1}{2}\dot{\phi}\sigma^2\nabla f(x) \\
& = \frac{1}{2}e^{\phi} \nabla[e^{-\phi}\sigma^2 \nabla f(x)],
\end{align*}
where we denote $\dot{\phi} = 2(\eta\rho + h)x/\sigma^2$ to get the second line. Therefore \rf{bou} is equivalent to
$$\mathcal{A}_1 g(x) = \frac{1}{2}e^{\phi} \nabla[e^{-\phi}\sigma^2 \nabla g(x)] = -1.$$
Integrating the above equation, we recover (13) of \cite{du2012power}:
\begin{equation} \label{inte}
\mathbb{E} \tau^x_{(a,b)}(\eta) = g(x) = 2\int_x^b dy \frac{e^{\phi(y)}}{\sigma^2(y)}\int_a^y e^{-\phi(z)}dz.
\end{equation}
The reader can check section 12.3 of \cite{weinan2019applied}  for the asymptotic analysis of \rf{inte} as $\eta \rightarrow 0$.\\
%In the high dimensional case, we aim to calculate the the escaping time of the SGD from one local minimizer $a$ of the potential function $l(w)$ to its nearest local minimizer $b$, the so-called Kramer âs problem. Let $w^*$ denote the saddle point with the minimal height among all saddle points that lie between $a$ and $b$. Note that the local form \rf{full} of the SGD is no longer available here. Instead, the SGD has the following form:
%$$dw_t = - \nabla l(w)dt + \sqrt{\eta C(w)}dB_t,$$
%Note that the diffusion part doesn't have the form of a scalar times the identity matrix, so we cannot quote the results in \cite{dai} directly.
We now investigate the general situation of $d>1$. We have only  asymptotic estimates on the exit time as the learning rate $\eta \rightarrow 0$. For this purpose, it is convenient to introduce the geometric reformulation of \rf{sim}. Suppose $B_t$ is the standard brownian motion,  recall that in local coordinates , the Riemannian Brownian motion $W_t$ with metric $\{g_{ij}\}$ has the following form (cf. e.g. \cite{hsu2002stochastic}):
\begin{equation} \label{rie}
dW_t^i = \sigma_j^i(W_t) dB^j_t + \frac{1}{2}b^i(W_t)dt,
\end{equation}
where $b^i(x) = g^{jk}(x)\Gamma_{jk}^{i}(x)$ and $\sigma_{ij}(x) = \sqrt{g^{ij}(x)}$. Comparing \rf{rie} with the martingale part of \rf{sim}, we define the inverse metric as $g^{ij} = ( \sigma_i + \rho_i (x^i)^2) \delta_{ij}$.  Then the metric  $g_{ij}$ is also a diagonal matrix. The christoffel symbols can be calculated under the new metric:
$$\Gamma_{jk}^{i}(x) = \rho_ix^i \cdot (\rho_i + \rho_i(x^i)^2),  \text{\ when\ }i=j=k,$$
and $\Gamma_{jk}^{i}(x) = 0$, otherwise. Denote the gradient vector filed of a smooth function $f(x)$ by $\nabla f(x)$, then
$$\nabla f(x) = (\partial_i f)g^{ij}\partial_j(x),$$
where $\partial_i(x) : = \frac{\partial}{\partial x_i}(x), \ 1\leq i \leq d,$ denotes the i-th coordinate tangent vector at $x \in \mathbb{R}^d$.
To emphasis the parameter $\eta$ that appears in the diffusion term in the power-law dynamic, we denote the dynamic and its corresponding exit time by $v_t(\eta)$ and $\tau^x_r (\eta)$.
Let $f_{\eta}(x) : = - \sum_i \frac{h_i}{2\rho_i} \log ( \sigma_i + \rho_i (x^i)^2) - \frac{\sqrt{\eta}}{2} \rho_i (\frac{\sigma_i}{2}(x^i)^2 + \frac{\rho_i}{3}(x^i)^3)$, then $v_t(\eta)$ in \rf{sim} can be seen as a diffusion process under the new metric:
\begin{equation} \label{geo}
  dv_t(\eta) = \nabla f_{\eta}(v_t) dt + \sqrt{\eta}dW_t,
\end{equation}
where $W_t$ is the Riemannian Brownian motion by \rf{rie}
and $0$ is a local minima of the limit function: $f_{lim}(x) :=\lim_{\eta \rightarrow 0} f_{\eta}(x) = - \sum_i \frac{h_i}{2\rho_i} \log ( \sigma_i + \rho_i (x^i)^2)$.
Note that both the drift term and the diffusion term are intrinsically defined with respect to the metric $\{g_{ij}\}_{1 \leq i, j \leq d}$. By large deviation theory, the rate function $I_{\eta}(\phi)$ of a path $\phi: [0,T] \rightarrow \mathbb{R}^d$ is:
$$I_{lim}(\phi) = \frac{1}{2} \int_0^T \norm{\dot{\phi}(t) - \nabla f_{lim}(\phi(t))}^2 dt + 2[f_{lim}(\phi(T)) - f_{lim}(\phi(0))],$$
where the norm is with respect to the Riemannian metric $g_{ij}$. It follows that the quasi-potential of the ball $B(0,r)$ is given by
$$\bar{f}_{lim} = 2[\inf_{\zeta = (\zeta^1,\dots,\zeta^d) \in \partial B(0,r)}f_{lim}(\zeta) - f_{lim}(0)].$$
%Recall the definition of Hessian of function $f(x)$:
%$$H_{ij}(f_{\eta}(x)) = \partial_i \partial_j f_{\eta}(x) - \Gamma_{ij}^{k}\partial_k f_{\eta}(x).$$
%From theorem 3.1 of \cite{bf}, the MFPT $\tau$ from a local minima $a$ (in our case, $a=0$) to a saddle point %$b$ with the only negative eigenvalue $\lambda_b$ of \rf{geo} as $\eta \rightarrow 0$ satisfies:
%$$\mathbb{E} \tau \simeq \frac{2\pi}{|\lambda_b|}\sqrt{\frac{|\det H_{ij}(b)|}{|\det H_{ij}(a)|}} \exp$$
By theorem 2.2 and corollary 2.4 of \cite{Nils}, if 0 is the only local minima of $f_{lim}$ in $B(0,r)$, then there exists a constant $C >0$ such that
\begin{align*}
\lim_{\eta \rightarrow 0} \eta  \log \mathbb{E} \tau^x_r (\eta)    & = C \cdot 2 [\inf_{\zeta = (\zeta^1,\dots,\zeta^d)  \in \partial B(0,r)} f_{lim}(\zeta) - f_{lim}(0)]\\
& =  C \cdot \inf_{\zeta = (\zeta^1,\dots,\zeta^d)  \in \partial B(0,r)} \sum_i - \frac{h_i}{\rho_i} \log [1 + \frac{\rho_i}{\sigma_i} (\zeta^i)^2].
\end{align*}
The proof is completed.

\end{proof}
\begin{remark}
When the dimension $d=1$, we can get similar results with a precise prefactor by applying semiclassical approximation to the integral (\ref{inte}), see \cite{kolokoltsov2007semiclassical}. Taking exponential of \rf{div}, it's obvious that the leading order of the average exit time is of the  power law form with respect to the radius $r$.
\end{remark}

\section{ First exit time: from continuous to discrete}

In this section, we compare the exit times  of the continuous power-law dynamic \rf{new1} and its discretization:
\begin{equation} \label{disc}z_{k+1} = z_k + \epsilon \mu(z_k) + \epsilon \epsilon_k ,\ \ \epsilon_k \sim \mathbb{N}(0,\sigma^2(z_k)).
\end{equation}
Note that the first exit time of the discretized dynamic \rf{disc} is an integer that measures how many steps it takes to escape from the ball, thus the $K$ time steps correspond to $K\epsilon$ amount of time. In this view point, the comparison can help guide machine learning algorithm provided the time interval  $\epsilon$ coincides with the learning rate $\eta$. However, since in power law dynamic the covariance matrix $\sigma^2(w_k)$ contains $\eta$, for
the convenience of theoretic discussion,  we should temporarily distinguish $\epsilon$ with $\eta$ before we arriving at the conclusion.

To shorten the length of the article, we shall confine ourselves in the situation of $d=1.$ Assume the local minima is $0$.
  Let $\tau_{r}^0$ be the first exit time from the ball $B(0,r)$  of the one dimensional continuous power-law dynamic in \rf{new1}  and let $\bar{\tau}_{r}^0$ be the corresponding first exit time of the discretized dynamic \rf{disc}, both starting from $0$. Then we have the following comparison of $\mathbb{P}[\bar{\tau}_{r}^0 > K]$ with the corresponding quantities related to the first exit times of the continuous dynamic.
 \begin{theorem} \label{discrete escape}
 Suppose $\delta, \bar{\delta} > 0$ and satisfy $\delta+ \bar{\delta} < r$, given a large integer $K,$ we have
 \begin{align*}
\mathbb{P}&[\tau_{r-\delta}^0 > K\epsilon] - \frac{4}{3\delta^2}\frac{E(\epsilon)K}{c\epsilon} \leq \mathbb{P}[\bar{\tau}_{r}^0 > K]\\
 &\leq \mathbb{P}[\tau^0_{r+\delta+\bar{\delta}} > K\epsilon] + \frac{4}{3\delta^2}\frac{E(\epsilon)K}{c\epsilon} + 1 - (1-\frac{C(\epsilon,\eta)}{\bar{\delta}^4})^K,
\end{align*}
where $E(\epsilon) \sim O(\epsilon^2)$ and $C(\epsilon,\eta) \sim O(\epsilon)$ as $\epsilon \rightarrow 0.$
%$\eta \rightarrow 0$.
\end{theorem}
%\begin{remark}
%The bound obtained in this theorem is nonvacuous only when $(1-\frac{C(\epsilon,\eta)}{\bar{\delta}^4})^K$ is close to one.
%\end{remark}
\begin{proof}
For our purpose we introduce an interpolation process $z_t$ as follows:
%To put the discrete stochastic gradient descent and its diffusion limit on an equal footing, we introduce the interpolation process $z_t \in \mathbb{R}$ from the k-th step to (k+1)-th step:

\begin{align}\label{dis}
dz_t = - h z_{k} dt + \sqrt{\eta \sigma + \eta \rho (z_{k})^2}dB_t,\ \ t \in [(k-1)\epsilon,k\epsilon),
\end{align}
where $\epsilon$ is the discretization step size. More precisely,
the drift coefficient and the diffusion coefficient of \rf{dis} will remain unchanged when $t \in [(k-1)\epsilon,k\epsilon)$ for each $1\leq k\leq K.$ Note that if we rewrite $\sqrt{\eta \sigma + \eta \rho (z_{k})^2}$ as $\sigma(z_{k}),$ \rf{dis} is expressed as
\begin{align}
dz_t = - h z_{k} dt + \sigma(z_{k}) dB_t, \ \ t \in [(k-1)\epsilon,k\epsilon),
\end{align}
which reduced to \rf{disc} when $t=k\epsilon$ for each $k.$

We shall adopt a similar strategy as in \cite{nguyen2019first} to transfer from the average exit time of the power-law dynamic $v_t$ to its discretization $z_t$. Below in the discussion we follow also the notations in the previous sections. Roughly speaking, the proof can be divided into two steps:

(i) ~~ Fix the number of iteration steps as K, prove that
$$\mathbb{P}((z_\epsilon,\dots,z_{K\epsilon})- (v_\epsilon,\dots,v_{K\epsilon}) \notin B_{\delta}) \leq \bar{\epsilon},$$

where $\bar{\epsilon}$ is a small positive constant to be determined, and\\ $B_{\delta} = \underbrace{B(0,\delta) \times \cdots \times B(0,\delta)}_K$ is the hyper-cube of radius $\delta > 0$. This can be down by bounding the $W_2$-distance of $v_{k\epsilon}$ and $z_{k\epsilon}$ for $1\leq k \leq K$. Let
$$A_{\delta} = \underbrace{B(0,r+\delta) \times \cdots \times B(0,r+\delta)}_K,$$
where $r > |\delta| > 0$. For simplicity, denote the exit time of the power-law dynamic from $A_{\delta}$ by $\tau_{\delta}$. For the interpolation process $z_t$, we denote the corresponding exit time (an integer) with a bar above it: $\tau_{\delta} \rightarrow \bar{\tau_{\delta}}$. Then,
\begin{align} \label{fun1}
\mathbb{P}[(v_\epsilon,\dots,v_{K\epsilon})& \in A_{-\delta}] - \bar{\delta} \leq \mathbb{P}[\bar{\tau_{0}} > K]\\ \nonumber
& \leq \mathbb{P}[(v_\epsilon,\dots,v_{K\epsilon}) \in A_{\delta}] + \bar{\epsilon}.
\end{align}
Note that the event $\{\bar{\tau_{\tau}} > K\}$ indicates that the interpolation process $z_t$ remains in $A_{\delta}$ when $t \leq K\epsilon $.

 (ii)~ Step 1 guarantees that if $v(t)$ is trapped in a ball with a different size when $t =\epsilon, 2\epsilon, \dots, K\epsilon$, then the interpolation process $z(t)$ is also trapped in a ball. However,
$$\mathbb{P}[(v_\epsilon,\dots,v_{k\epsilon}) \in A_{\delta}] > \mathbb{P}[\tau_{\delta} > K\epsilon],
$$
since $v(t)$ may drift outside the ball when $t \in [(k-1)\epsilon,k\epsilon)$ for $ 1\leq k \leq K$. We define this ~'anomalous'~ random event by
$$R := \{\max_{0\leq k \leq K-1}\sup_{t \in (k\epsilon , (k+1)\epsilon)}\norm{v_t - v_{k\epsilon}} > \bar{\delta}\}.$$
Then obviously,
\begin{align} \label{fun2}
 \mathbb{P}[(v_\epsilon,\dots,v_{k\epsilon})& \in A_{\delta}] \leq \\ \nonumber
  \mathbb{P}[\tau_{\delta+\bar{\delta}}> k\epsilon]& + \mathbb{P}[(v_\epsilon,\dots,v_{k\epsilon}) \in R^c],
\end{align}
where $R^c$ denotes the complement of the event $R$. We would expect that the probability of $\{(v(\epsilon),\dots,v(k\epsilon)) \in R^c\}$ to be small if the diffusion coefficient of the dynamic is bounded, in which case we can apply Gaussian concentration results. However, the diffusion part of the power-law dynamic is not bounded, so additional technical issue should be taken care of.

Now we introduce the same form of coupling as in the previous sections between $(v_t,z_t)$ for $t \in [(k-1)\epsilon,k\epsilon), \ \forall 1 \leq k \leq K$ . Following the notations in the proof of theorem \ref{thm:uni}, we set the $\alpha_s(v,z)$ and $\beta_s(v,z)$ of the infinitesimal generator $\mathcal{A}_s$ of the coupling as:
\begin{equation}
\label{coupling}
\alpha_s(v,z) = \begin{pmatrix} \sigma(v)\sigma^T(v), & \sigma(v)\sigma(z_{k\epsilon}) \\ \sigma(z_{k\epsilon})\sigma(z), & \sigma(z_{k\epsilon})\sigma^T(z_{k\epsilon}) \end{pmatrix} ,\ \ \beta_s(v,z)=\begin{pmatrix} -hv \\ -hz_{k\epsilon} \end{pmatrix}.
\end{equation}
The remainder proof of the theorem will be accomplished by three lemmas.   We prove first the following lemma for the one dimensional decoupled dynamic \rf{sim}:
\begin{lemma}\label{lemma1}

	Suppose that the coefficients of \rf{sim} satisfy:
	\begin{equation} \label{assumption}
	\{\eta\rho, \rho \} \leq h \leq \frac{1}{\epsilon},\end{equation} (which is fulfilled in SGD algorithm for large batch size and small learning rate).
	Let $(v(t),z(t))$ be the coupling process defined by \rf{coupling}, and $v(0) = z(0) = x$.
	 When $t= K\epsilon$, the Warsserstein distance between the marginal distribution $p_K^z$ of $z(t)$ and the marginal distribution $p_K^v$ of the power-law dynamic $v(t)$ is bounded by
	\begin{equation}
	\mathcal{W}_2^2(p_K^v,p_K^z) \leq  \frac{4}{3}\frac{E(\epsilon)}{c\epsilon},
	\end{equation}
	where $c = 2h - \eta \rho > 0$. Moreover, $E(\epsilon) > 0$ is independent of the number of steps K and is of order $O(\epsilon^2)$ when the time interval $\epsilon \rightarrow 0$.
\end{lemma}
%\begin{remark}
%Assumption \rf{assumption} holds in analysing SGD algorithm for large batch size and small learning rate.
%\end{remark}
\begin{proof}[Proof of Lemma \ref{lemma1}]
Denote the transition probability of $v_t$ and $z_t$ from time $(k-1)\epsilon$ to $(k-1)\epsilon +t$ by $p^v_{(k-1)\epsilon +t}(v_{(k-1)\epsilon},\cdot)$ and  $p^z_{(k-1)\epsilon +t}(z_{(k-1)\epsilon},\cdot)$ respectively. Let $\mathcal{A}_s$ act on the function $r(v,z) : = \norm{v-z}^2$ and use the Gronwall's inequality, we can deduce the following recursion inequality for $1 \leq k \leq K$:
\begin{align*}
\mathcal{W}_2^2 (p_{k}^v(\cdot), p_{k}^z(\cdot)) & \leq \int_{\mathbb{R}^d \times \mathbb{R}^d} \mathcal{W}_2^2 (p^v_{(k-1)\epsilon} +t(v_{(k-1)\epsilon},\cdot), p^z_{(k-1)\epsilon}(z_{(k-1)\epsilon},\cdot)) \\ &d\pi(p_{k-1}^v(v_{(k-1)\epsilon}) , p_{k-1}^z(z_{(k-1)\epsilon})) \\
& \leq e^{-c\epsilon}\mathcal{W}_2^2 (p_{k-1}^v(\cdot),p_{k-1}^z(\cdot)) +  \\
&\  \int_{\mathbb{R}} \frac{1}{2}\epsilon^3 \rho (8\eta \rho + 4h)\norm{x}^2 dp_{k-1}^v(x) +\\
&\ \int_{\mathbb{R}} [\frac{1}{3}\epsilon^3h(4\eta \rho + 2h)\rho + \frac{1}{2}\epsilon^2 \sigma (8\eta \rho + 3h)]\norm{y}^2 dp_{k-1}^z(y) +\\
&\ \epsilon^2 \eta \rho(4\eta \rho + 2h) + \frac{1}{2}\epsilon^2 \sigma(8\eta \rho + 3h),
\end{align*}
where $c = 2h - \eta \rho > 0$ and $\pi(p_{k-1}^v(\cdot) , p_{k-1}^z(\cdot))$ is the optimal coupling between $p_{k-1}^v(\cdot)$ and $p_{k-1}^z(\cdot)$. ( Note that in our context the optimal coupling always exists, see e.g. Proposition 1.3.2 and Theorem 2.3,3 in \cite{wang}.)

For the interpolation process $z_t$ starting at $y_0$, by Ito's formula,
 we have the following estimate:
$$\int_{\mathbb{R}^d} \norm{y}^2 dp_k^z(y) \leq e^{-(2h - \eta \rho)k\epsilon}(y_0^2 - \frac{\eta \rho}{2h - \eta \rho}) + \frac{\eta \rho}{2h - \eta \rho}.$$

Similarly, the second moment of the continuous dynamic $v_t$ starting at $x_0$ can be bounded by:
\begin{equation} \label{second}
\mathbb{E}\norm{v_t}^2 \leq  e^{-(2h - \eta\rho)t}(x_0^2 - \frac{\eta\rho}{2h - \eta\rho}) + \frac{\eta\rho}{2h - \eta\rho}.
\end{equation}
Notice that the distance between $v_t$ and $z_t$ is zero at initialization, then by applying the recursion relation from $k=1$ to $k =K$, we conclude that there exits  $E(\epsilon)>0$, such that
$$\mathcal{W}_2^2(p_K^v,p_K^x) \leq  \frac{4}{3}\frac{E(\epsilon)}{c\epsilon},$$
where $E(\epsilon) \sim O(\epsilon^2)$ is independent of K, which completes the proof of  Lemma \ref{lemma1}.
\end{proof}
By the definition of the $W_2$-distance,
\begin{align} \nonumber
\mathbb{P}((z_\epsilon,\dots,z_{K\epsilon})- (v_\epsilon,\dots,v_{K\epsilon}) \notin B_{\delta}) &\leq \sum_{k=1}^{K} \frac{W_2^2(p_k^v, p_k^z)}{\delta^2} \\  \label{close}
& \leq \frac{4}{3\delta^2}\cdot\frac{E(\epsilon)K}{c\epsilon}.
\end{align}
 Below we denote the above right hand side $\frac{4}{3\delta^2}\frac{E(\epsilon)K}{c\epsilon}$ by $\bar{\epsilon}$.
For the second step, from \rf{fun1} and \rf{fun2}, it follows that
\begin{equation} \label{relation}
\mathbb{P}[\tau_{-\delta} > K\epsilon] -  \bar{\epsilon} \leq \mathbb{P}[\bar{\tau}_{0} > K] \leq   \mathbb{P}[\tau_{\delta+\bar{\delta}}> K\epsilon] + \mathbb{P}[(v_\epsilon,\dots,v_{K\epsilon}) \in R^c] + \bar{\epsilon}.
\end{equation}
Therefore, we are left to estimate $\mathbb{P}[(v_\epsilon,\dots,v_{K\epsilon}) \in R^c]$. Under the condition $h > \frac{1}{2}\eta \rho$, we have the following lemma:
\begin{lemma} \label{fourth moment} Let $\delta > 0$ be fixed.
Conditioning on the event that $v_t$ is inside $B(0,b+\delta)$ when $t = k\epsilon$ for all $ 1 \leq k \leq K$, we have
$$\mathbb{E}\sup_{s \in [k\epsilon,(k+1)\epsilon)} (v_s)^4 \leq D(\eta, \rho, \epsilon),$$
where $D(\eta, \rho, \epsilon) : = [(2+2/\delta)\frac{(b+\delta)^2 \eta \rho}{2h - \eta \rho} + (5+  \frac{1}{\delta}) (\eta \sigma)^2\epsilon]\exp\{12(1+\delta) \eta \rho\epsilon\}.$
%and $\delta$ is an arbitrary positive constant.
\end{lemma}
\begin{remark}
 The above lemma tells us that the fourth moment won't change too much if the time interval $\epsilon$ is small. Intuitively, Since the martingale part of $v_t$ is $\sqrt{\eta \sigma + \eta \rho v_t^2}dB_t$, if $|v_t|$ is bounded, then by the time change theorem, we know that the marginal distribution of the martingale part behaves like a scaled Gaussian.
\end{remark}
\begin{proof}[Proof of Lemma \ref{fourth moment}]
By Ito's formula,
$$d(v_t)^2 =  \eta \sigma dt - (2h - \eta \rho)(v_t)^2 dt + 2v_t\sqrt{\eta \sigma + \eta \rho_i(v_t)^2} dB_t.$$
Define a local martingale $M_t : = 2 \sum \int_0^t e^{(2h - \eta \rho)s}\sqrt{\eta \sigma + \eta \rho(v_s)^2}v_s dB_s $, then by Gronwall's inequality,
\begin{align*}
(v_t)^2 &\leq e^{-(2h - \eta \rho)t} (v_0)^2 + e^{-(2h - \eta \rho)t}M_t + \\ &\ \ e^{-(2h - \eta \rho)t} \int_0^t e^{(2h - \eta \rho)s}\eta \sigma ds .
\end{align*}
Let$$S_t : = \mathbb{E}\sup_{s \in [k\epsilon,k\epsilon +t]} (v_s)^4, \ \ t \in  [0,\epsilon).$$
Then, for the fixed $\delta > 0$,
\begin{align*}
S_t &\leq (1+\delta)e^{-2(2h - \eta \rho)t} \mathbb{E} \sup_{s \in [0,t]} (M_s)^2 + (2+2/\delta) e^{-2(2h - \eta \rho)t} \mathbb{E} (v_{k\epsilon})^4 \\ &\ \ \ \ + (2+2/\delta)e^{-2(2h - \eta \rho)t} \mathbb{E}(\int_0^t e^{(2h - \eta \rho)s} \eta \sigma ds)^2.
\end{align*}
Applying Doob's inequality, we get
\begin{align*}
S_t \leq 2(1 + \delta) e^{-2(2h - \eta \rho)t}& \mathbb{E} M_t^2 + (1+\frac{1}{\delta})\frac{e^{-2(2h - \eta \rho)t}-1}{(\eta \rho-2h)t}\\
&\cdot \mathbb{E}\int_0^t (\eta \sigma )^2ds + (2+2/\delta) e^{-2(2h - \eta \rho)t} \mathbb{E} (v_{k\epsilon})^4.
\end{align*}
 Now, Ito's isometry implies that
\begin{align*}
S_t &\leq 12(1+\delta) \eta \rho \int_0^t S_s ds + (2+2/\delta) e^{-2(2h - \eta \rho)t} \mathbb{E} (v_{k\epsilon})^4\\
&\ \ +  (5+  \frac{1}{\delta}) (\eta \sigma)^2t \\
&\leq [(2+2/\delta)\mathbb{E} (v_{k\epsilon})^4 + (5+  \frac{1}{\delta}) (\eta \sigma)^2t]\exp\{12(1+\delta) \eta \rho t\},
\end{align*}
where we used Gronwall's inequality to derive the last line. For all  $ 1 \leq k \leq K$,  by taking $x_0 = 0$ in \rf{second}, we get
$$\mathbb{E}\{(v_{k\eta})^4 \mathbb{I}_{(v_\epsilon,\dots, v_{K\epsilon}) \in A_{\delta}}\} \leq \frac{(b+\delta)^2 \eta \rho}{2h - \eta \rho}.$$
Then, conditioning on the event that $\{v_{k\epsilon} \in B(0,b+\delta)\}$, we have
\begin{align*}
\mathbb{E}\sup_{s \in [k\epsilon,(k+1)\epsilon)} (v_s)^4 &\leq [(2+2/\delta)\frac{(b+\delta)^2 \eta \rho}{2h - \eta \rho} + (5+  \frac{1}{\delta}) (\eta \sigma)^2\epsilon]\\
&\cdot \exp\{12(1+\delta) \eta \rho\epsilon\}: = D(\eta, \rho, \epsilon),
\end{align*}
which completes the proof of Lemma \ref{fourth moment}.
\end{proof}

\begin{lemma}\label{lemma3}
Let $\bar{\delta}$ be a positive constant, then for every $k \in [0,\dots,K-1]$,
$$\mathbb{P}(\sup_{t \in [k\epsilon, (k+1)\epsilon)} \norm{v_t - v_{k\epsilon}} > \bar{\delta}) \leq \frac{C(\epsilon,\eta)}{\bar{\delta}^4},$$
where $C(\epsilon,\eta) \sim O(\epsilon)$ when  $\epsilon \rightarrow 0$.
\end{lemma}

\begin{proof}[Proof of Lemma \ref{lemma3}]Let $r(x) = \norm{x-v_{k\epsilon}}^2$, then by Ito's lemma,
\begin{align*}
d r(v_t) &\leq -2h v_t( v_t - v_{k\epsilon}) dt + 2\sqrt{\eta \sigma + \eta \rho (v_t)^2}(v_t - v_{k\epsilon})dB_t + ( \eta \sigma + \eta \rho (v_t)^2 ) dt \\
& \leq -2h ( v_t - v_{k\epsilon})^2 dt + 3h (v_{k\epsilon})^2 dt + \eta \sigma dt\\
& \ \ \ + (\eta \rho + h) (v_t)^2 dt + 2\sqrt{\eta \sigma + \eta \rho (v_t)^2}(v_t - v_{k\epsilon})dB_t,
\end{align*}
and obviously we know that $r(v_{k\epsilon}) = 0$. Let
$$M_t : = 2 \int_{k\epsilon}^t e^{2h s}\sqrt{\eta \sigma + \eta \rho (v_s)^2}(v_s - v_{k\epsilon})dB_s ,$$
we have
$$r(v_t) \leq e^{-2h t}M_t + e^{-2h t} \int_{k\epsilon}^t e^{2h s}[3h (v_{k\epsilon})^2 + \eta \sigma + (\eta \rho + h) (v_s)^2]ds .$$
Let $S_t : = \mathbb{E} \sup_{s \in [k\epsilon,t]} r^2(v_s),$ then
\begin{align*}
S_t  \leq 2e^{-4h (t-k\epsilon)}&\mathbb{E} \sup_{s \in [k\epsilon,t]} (M_s)^2 + 2e^{-4h (t-k\epsilon)}\\
&\cdot \mathbb{E} (\int_{k\epsilon}^t e^{2h s}[3h (v_{k\epsilon})^2 + \eta \sigma + (\eta \rho + h) (v_s)^2]ds)^2\\
 \leq 4e^{4h (t-k\epsilon)}&\mathbb{E}(M_t)^2 + \frac{1-e^{-4h (t-k\epsilon)}}{h}\mathbb{E}\int_{k\epsilon}^t  [3h (v_{k\epsilon})^2 + \eta \sigma + (\eta \rho + h) (v_s)^2]^2 ds \\
& \leq 16\mathbb{E} \int_{k\epsilon}^t (\eta \sigma + \eta \rho (v_s)^2)(v_s - v_{k\epsilon})^2 ds + \\ &\ \ \ \ \ \frac{1-e^{-4h (t-k\epsilon)}}{h}\mathbb{E}\int_{k\epsilon}^t  [3h (v_{k\epsilon})^2 + \eta \sigma + (\eta \rho + h) (v_s)^2]^2 ds.
\end{align*}
Let $t = (k+1)\epsilon$, and denote the above right hand side as $C(\epsilon,\eta)$. Since
$\mathbb{E}|v_s|^2 \leq \sqrt{\mathbb{E}|v_s|^4},$ by Lemma \ref{fourth moment},
$$C(\epsilon,\eta) \sim O\left(\frac{\exp(\eta \epsilon) -1}{\eta} \right) = O(\epsilon).$$
Therefore,
$$\mathbb{P}(\sup_{t \in [k\epsilon, (k+1)\epsilon)} \norm{v(t) - v(k\epsilon)} > \bar{\delta}) \leq \frac{\mathbb{E}S_{(k+1)\epsilon}}{\bar{\delta}^4} \leq \frac{C(\epsilon,\eta)}{\bar{\delta}^4}.$$ The proof of Lemma \ref{lemma3} is completed
\end{proof}
From the previous lemma, we can easily deduce the following bound on $\mathbb{P}[(v_\epsilon,\dots,v_{K\epsilon}) \in R^c]$:
\begin{equation} \label{bad event}
 \mathbb{P}[(v_\epsilon,\dots,v_{K\epsilon}) \in R^c] \leq 1 - (1 - \frac{C(\epsilon,\eta)}{\bar{\delta}^4})^K.
\end{equation}
Combing \rf{bad event} and \rf{relation} , we have finished the proof of the theorem.
\end{proof}

\begin{remark}
Theorem \ref{discrete escape} discussed only the 1-dimensional case. For high dimensional case, there have been some intuitive discussion in machine learning literature \cite{xie2020diffusion}.
Roughly speaking, the escaping path will concentrate on the critical paths, i.e., the paths on the direction of the eigenvector of the Hessian,  when the noise is much smaller than the barrier height with high probability.  If there are multiple parallel exit paths, the total exiting rate, i.e., the inverse of the expected exit time,  equals to the sum of the first exiting rate for every path (cf. Rule 1 in \cite{xie2020diffusion}).
\end{remark}

%\begin{remark}
%The bound obtained in this theorem is nonvacuous only when $(1-\frac{C(\epsilon,\eta)}{\bar{\delta}^4})^K$ is close to one. Theorem \ref{discrete escape} is for 1-dimensional case. For high dimensional case, the escaping path will concentrate on the critical paths (i.e., the paths on the direction of the eigenvector of the Hessian.) when the noise is much smaller than the barrier height with high probability \cite{xie2020diffusion}. If there are multiple parallel exit paths, the total exiting rate (i.e., the inverse of the expected exit time) equals to the sum of the first exiting rate for every path (cf. Rule 1 in \cite{xie2020diffusion}).
%\end{remark}
%In conclusion, theorem \ref{escaping} indicates the asymptotic order of the first exit time with respect to the radius $r$ as $\eta \rightarrow 0$.   Theorem \ref{discrete escape} demonstrated that the probability distribution of the first exit time from a ball of the discretized power-law dynamic can be bounded by the continuous dynamic of exiting a ball with different size. For example, we can take $\bar{\delta}$ in theorem \ref{discrete escape} to be $r$, then as $r \rightarrow \infty$, the first exit time of both discrete and continuous dynamics are of same order with respect to the radius $r$.

\begin{acknowledgements}
We had pleasant collaboration with Shiqi Gong, Huishuai Zhang, and Tie-Yan Liu on the research of power-law dynamics in machine learning\cite{meng22020dynamic,meng2020dynamic}. We thank them for their  contributions in the previous work and their  comments on this work, especially based on the empirical observations during our previous collaboration.
\end{acknowledgements}



%\input{powerlaw1.bbl}
%\bibliographystyle{plain}
%\bibliography{powerlaw1}
%\bibliographystyle{plain}

%\bibliography{powerlaw1}

\end{document}